\documentclass[a4paper]{article}      

\usepackage[colon,sort,round]{natbib}            

\usepackage{psfrag}
\usepackage{eufrak}
\usepackage{nicefrac}
\usepackage{amsmath}
\usepackage{latexsym}
\usepackage{nomencl}
\usepackage{array}
\usepackage{acronym}

\usepackage{algorithm}
\usepackage{algorithmic}

\newtheorem{theorem}{Theorem}
\newtheorem{lemma}[theorem]{Lemma}

\newcommand{\E}{\ensuremath{\mathbf{E}}}

\newcommand{\alg}{\ensuremath{\mathrm{KUBE}}}
\newcommand{\algKUBE}{\ensuremath{\mathrm{KUBE}}}

\newcommand{\argmax}{\mathop{\rm argmax}}
\newcommand{\cmin}{ c_{\mathrm{min}} }
\newcommand{\cmax}{ c_{\mathrm{max}} }

\begin{document}

\title {Knapsack based Optimal Policies for Budget--Limited Multi--Armed Bandits}

\author{
Long Tran--Thanh\footnote{School of Electronics and Computer Science, University of Southampton, UK. Contact:\texttt{ltt08r@ecs.soton.ac.uk}.} \\ 
Archie Chapman\footnote{The University of Sydney Business School Sydney, Australia.} \and
Alex Rogers\footnotemark[1] \\ 
Nicholas R. Jennings\footnotemark[1] \\
}

\date{}
\maketitle

\begin{abstract}
\noindent
In budget--limited multi--armed bandit (MAB) problems, the learner's actions are costly and constrained by a fixed budget.  
Consequently, an optimal exploitation policy may not be to pull the optimal arm repeatedly, as is the case in other variants of MAB, but rather to pull the sequence of different arms that maximises the agent's total reward within the budget. 
This difference from existing MABs means that new approaches to maximising the total reward are required.  
Given this, we develop two pulling policies, namely: (i) \alg; and (ii) fractional \alg.
Whereas the former provides better performance up to $40\%$ in our experimental settings, the latter is computationally less expensive.  
We also prove logarithmic upper bounds for the regret of both policies, and show that these bounds are asymptotically optimal (i.e. they only differ from the best possible regret by a constant factor). 
\end{abstract}

\noindent
\section{Introduction}
\label{Section:intro}
\noindent 
The standard multi--armed bandit (MAB) problem was originally proposed by \citeauthor{Robbins1952} (\citeyear{Robbins1952}), and presents one of the clearest examples of the trade--off between \emph{exploration} and \emph{exploitation} in reinforcement learning.
In the standard MAB problem, there are $K$ arms of a single machine, 
each of which delivers rewards that are independently drawn from an unknown distribution when an arm of the machine is pulled. 
Given this, an agent must choose which of these arms to pull. 
At each time step, it pulls one of the machine's arms and receives a reward or payoff. 
The agent's goal is to maximise its return; that is, the expected sum of the rewards its receives over a sequence of pulls. 
As the reward distributions differ from arm to arm, 
the goal is to find the arm with the highest expected payoff as early as possible, 
and then to keep playing using that best arm.
However, the agent does not know the rewards for the arms, so it must sample them in order to learn which is the optimal one.
In other words, in order to choose the optimal arm (exploitation) the agent first has to estimate the mean rewards of all of the arms (exploration).
In the standard MAB, this trade--off has been effectively balanced by decision--making policies such as \emph{upper confidence bound} (UCB) and $\epsilon_{n}$--\emph{greedy} \citep{AuerEtal2002}.

However, this MAB model gives an incomplete description of the sequential decision--making problem facing an agent in many real--world scenarios. 
To this end, a variety of other related models have been studied recently, and, 
in particular, a number of researchers have focused on MABs with budget constraints, 
where arm--pulling is costly and is limited by a fixed budget \citep{BubeckEtal2009,GuhaMunagala2007,Antos2008}.  
In these models, the agent's exploration budget limits the number of times it can sample the arms in order to estimate their rewards, which defines an initial exploration phase.  
In the subsequent cost--free exploitation phase, an agent's policy is then simply to pull the arm with the highest expected reward.  
However, in many settings, it is not only the exploration phase, but the exploitation phase that is also limited by a cost budget.
To address this limitation, a new bandit model, the \emph{budget--limited MAB},
was introduced (Tran-Thanh \emph{et al.}~\citeyear{Tran-ThanhEtAl2010}). 
In this model, pulling an arm is again costly, but crucially \emph{both} the exploration and exploitation phases are limited by a \emph{single budget}. 
This type of limitation is well motivated by several real--world applications. 
For example, in many wireless sensor network applications, 
a sensor node's actions, such as sampling or data forwarding, consume energy, and therefore the number of actions is limited by the capacity of the sensor's batteries (Padhy \emph{et al.}~\citeyear{PadhyEtAl2010}). 
Furthermore, many of these scenarios require that sensors learn the optimal sequence of actions that can be performed, with the goal of maximising the long term value of the actions they take \citep{Tran-ThanhEtAl2011}.
In such settings, each action can be considered as an arm, with a cost equal to the amount of energy needed to perform that task.
Now, because the battery is limited, both the exploration (i.e.~learning the rewards tasks) and exploitation (i.e.~taking the optimal actions given reward estimates) phases are budget limited.

Against this background, Tran-Thanh \emph{et al.}~(\citeyear{Tran-ThanhEtAl2010}) showed that the budget--limited MAB cannot be derived from any other existing MAB model, 
and therefore, previous MAB learning methods are not suitable to efficiently deal with this problem.  
Thus, they proposed a simple budget--limited $\varepsilon$--\emph{first} approach for the budget--limited MAB. 
This splits the overall budget $B$ into two portions, 
the first $\varepsilon B$ of which is used for exploration, and the remaining $(1-\varepsilon)B$ for exploitation. 
However, this budget--limited $\varepsilon$--{first} method suffers from a number of drawbacks.
First, the performance of $\varepsilon$--first approaches depend on the value of $\varepsilon$ chosen. 
In particular, high values guarantee accurate exploration but inefficient exploitation, and \emph{vice versa}. 
Given this, 
finding a suitable $\varepsilon$ for a particular problem instance is a challenge, 
since settings with different budget limits or arm costs (which are not known beforehand) 
will typically require different values of $\varepsilon$. 
In addition, even with a good $\varepsilon$ value, 
the method typically provides poor efficiency in terms of minimising its performance regret 
(defined as the difference between its performance and that of the optimal policy), which is a standard measure of performance.
In particular, the regret bound that $\varepsilon$--first provides is $O\left(B^{\frac{2}{3}}\right)$, where $B$ is the budget limit, whereas the theoretical best possible regret bound is typically a logarithmic function of the number of pulls\footnote{ 
Note that in the budget--limited MAB, the budget $B$ determines the number of pulls. 
Thus, a logarithmic function of the number of pulls is also a logarithmic function of the budget.} 
\citep{LaiAndRobbins1985}.   

To address this shortcoming, in this paper we propose two new learning algorithms, called \alg\ (for knapsack--based upper confidence bound exploration and exploitation) and \emph{fractional \alg}, that do not explicitly separate exploration from exploitation.
Instead, they explore and exploit at the same time by adaptively choosing which arm to pull next, 
based on the current estimates of the arms' rewards.  
In more detail, at each time step, \alg\ calculates the best set of arms that provides the highest total upper confidence bound of the estimated expected reward, and still fits into the residual budget, using an unbounded knapsack model to determine this best set  \citep{KellererEtAl2004}.
However, since unbounded knapsack problems are known to be NP--hard, the algorithm uses an efficient approximation method taken from the knapsack literature, called the \emph{density--ordered greedy} approach, in order to estimate the best set \citep{KohliEtal2004}.    
Following this, \alg\ then uses the frequency that each arm occurs within this approximated best set as a \emph{probability} with which to {randomly} choose an arm to pull in the next time step. 
The reward that is received is then used to update the estimate of the upper confidence bound of the pulled arm's expected reward, 
and the unbounded knapsack problem is solved again.
The intuition behind this algorithm is that if we know the real value of the arms, then the budget--limited MAB can be reduced to an unbounded knapsack problem, where the optimal solution is to subsequently pull from the set of arms that forms the solution of the knapsack problem.
Given this, by randomly choosing the next arm from the current best set at each time step, the agent generates an accurate estimate of the true optimal solution (i.e. real best set of arms), and, accordingly, 
the sequence of pulled arms will converge to this optimal set.
%
%
In a similar vein, fractional \alg\ also estimates the best set of arms that provides the highest total upper confidence bound of the estimated expected reward at each time step, and uses the frequency that each arm occurs within this approximated best set as a probability to randomly pull the arms.
However, instead of using the density--ordered greedy to solve the underlying unbounded knapsack problem, fractional \alg\ relies on a computationally less expensive approach, namely the \emph{fractional relaxation based} algorithm \citep{KellererEtAl2004}. 
Given this, fractional \alg\ requires less computation than \alg.

To analyse the performance of \alg\ and its fractional counterpart in terms of minimising the regret, 
we devise proveably asymptotically optimal upper bounds on their \emph{performance regret}. 
That is, our proposed upper bounds differ from the best possible one only with a constant factor.
%
Following this, we numerically evaluate the performance of the proposed algorithms against a state--of--the--art method, namely the buget--limited $\varepsilon$--first approach, in order to demonstrate that our algorithms are the first that can achieve this optimal bound.
In addition, we show that \alg\ typically outperforms its fractional counterpart by up to $40\%$, however, this results in an increased computational cost (from $O\left(K\right)$ to $O\left(K\ln{K}\right)$).
Given this, the main contributions of this paper are:
\begin{itemize}
\item{We introduce \alg\ and fractional \alg, the first budget--limited MAB learning algorithms that proveably achieve a $O\left(\ln{B}\right)$ theoretical upper bound on the regret, where $B$ is the budget limit.}
\item{We demonstrate that with an increased computational cost, \alg\ outperforms fractional \alg\ in the experiments. We also show that while both algorithms achieve logarithmic regret bounds, the buget--limited $\varepsilon$--first approaches fail to do so.}
\end{itemize}
\noindent 
The paper is organised as follows: 
Next we describe the budget--limited MAB. 
We then introduce our two learning algorithms in Section~\ref{Section:online_learning}.  
In Section~\ref{Section:perf_analysis} we provide regret bounds on the performance of the proposed algorithms. 
Following this, Section~\ref{Section:numerical} presents an empirical comparison of \alg\ and its fractional counterpart with the $\varepsilon$--first approach.  
Section~\ref{Section:conclusion} concludes.

\section{Model Description}
\label{Section:model}
\noindent 
The budget--limited MAB model consists of a machine with $K$ arms, one of which must be pulled by the agent at each time step. 
By pulling arm $i$, the agent has to pay a pulling cost, denoted with $c_{i}$, and receives a non--negative reward drawn from a distribution associated with that specific arm.  
The agent has a cost budget $B$, which it cannot exceed during its operation time (i.e. the total cost of pulling arms cannot exceed this budget limit). 
Now, since reward values are typically bounded in real--world applications, we assume that each arm's reward distribution has bounded supports.  
Let $\mu_{i}$ denote the mean value of the rewards that the agent receives from pulling arm $i$.
Within our model, the agent's goal is to maximise the sum of rewards it earns from pulling the arms of the machine, with respect to the budget $B$. 
However, the agent has no initial knowledge of the $\mu_{i}$ of each arm $i$, 
so it must learn these values in order to deduce a policy that maximises its sum of rewards.  
Given this, our objective is to find the optimal pulling algorithm, 
which maximises the expectation of the total reward that the agent can achieve, without exceeding the cost budget $B$. 

Formally, let $A$ be an arm--pulling algorithm, giving a finite sequence of pulls. 
Let $N_{i}^{A}\left(B\right)$ be the random variable that represents the number of pulls of arm $i$ by $A$, with respect to the budget limit $B$. 
Since the total cost of the sequence $A$ cannot exceed $B$, we have:
\vspace{-0.6\baselineskip}
\begin{equation}
\label{eq:budget_limit_def}
P\left(\sum_{i}^{K}{N_{i}^{A}\left(B\right)c_{i}} \leq B\right) = 1.
\end{equation}
\normalsize
Let $G\left(A\right)$ be the total reward earned by using $A$ to pull the arms. 
The expectation of $G\left(A\right)$ is:
\vspace{-0.6\baselineskip}
\begin{equation}
\label{eq:reward_def}
\E \left[ G\left(A\right) \right] = \sum_{i}^{K}{\E \left[N_{i}^{A}\left(B\right)\right]\mu_{i}}.
\vspace{-0.2\baselineskip}
\end{equation}
\normalsize
Then, let $A^{*}$ denote an optimal solution that maximises the expected total reward, that is:
\vspace{-0.6\baselineskip}
\begin{equation}
\label{eq:optimal_policy_def}
A^{*} = \argmax_{A}{\sum_{i}^{K}{\E \left[N_{i}^{A}\left(B\right)\right]\mu_{i}}}.
\end{equation}
\normalsize
Note that in order to determine $A^{*}$, we have to know the value of $\mu_{i}$ in advance, which does not hold in our case. 
Thus, $A^{*}$ represents a theoretical optimum value, which is unachievable in general. 

Nevertheless, for any algorithm $A$, we can define the regret for $A$ as the difference between the expected cumulative reward for $A$ and that of the theoretical optimum $A^{*}$. 
More precisely, letting $R\left(A\right)$ denote the regret, we have:
\vspace{-0.3\baselineskip}
\begin{equation}
\label{eq:loss_function}
R\left(A\right)= \E\left[G\left(A^{*}\right)\right] - \E\left[G\left(A\right)\right].
\end{equation}
\normalsize
Given this, our objective is to derive a method of generating a sequence of arm pulls that minimises this regret for the class of MAB problems defined above.

\section{The Algorithms}
\label{Section:online_learning}
\noindent
Given the model described in the previous section, we now introduce two learning methods, \alg\ and fractional \alg, that efficiently deal with the challenges discussed in Section~\ref{Section:intro}.
Recall that at each time step of the algorithms, we determine the optimal set of arms that provides the best total estimated expected reward.
Due to the similarities of our MAB to unbounded knapsack problems when the rewards are known, 
we use techniques taken from the unbounded knapsack domain.
Thus, in this section, we first introduce the unbounded knapsack problem, and then show how to use knapsack methods in our algorithms.

\vspace{-0.1cm}
\subsection{The Unbounded Knapsack Problem}
\noindent
The unbounded knapsack problem is formulated as follows. 
A knapsack of weight capacity $C$ is to be filled with some set of $K$ different types of items.
Each item type $i\in K$ has a corresponding value $v_{i}$ and weight $w_{i}$, 
and the problem is to select a set that maximises the total value of items in the knapsack, 
such that their total weight does not exceed the knapsack capacity $C$. 
%
That is, the goal is to find the non--negative integers $\left\{x_{i}\right\}^{K}_{i=1}$ that maximise:
\begin{eqnarray}
&&\sum_{i=1}^{K}{x_{i}v_{i}}, \\
\nonumber
&\quad\mathrm{s. t.}\quad& \sum_{i=1}^{K}{x_{i}w_{i}} \leq C, \\
\nonumber
&&\forall i \in \{1,\dots,K\}: x_{i} \: \mathrm{integer}.
\end{eqnarray}
\normalsize
Note that this problem is a generalisation of the standard knapsack problem, in which $x_{i} \in \{0,1\}$; that is, each item type contains only one item, and we can either choose it or not.
The unbounded knapsack problem is \emph{NP}--hard. 
However, near--optimal approximation methods have been proposed to solve it (a detailed survey can be found in \citep{KellererEtAl2004}). 
Among these approximation methods, a simple, but efficient approach is the \emph{density--ordered greedy} algorithm, 
and here we make use of this method.
In more detail, the density--ordered greedy algorithm has $O\left(K \log{K} \right)$ computational complexity, 
where $K$ is the number of item types \citep{KohliEtal2004}.   
This algorithm works as follows. 
Let $\nicefrac{v_{i}}{w_{i}}$ denote the \emph{density} of type $i$.
To begin, the item types are sorted in order of their density, 
which is an operation of $O\left(K \log K \right)$ computational complexity. 
Next, in the first round of this algorithm, 
as many units of the highest density item are selected as is feasible without exceeding the knapsack capacity. 
Then, in the second round, the densest item of the remaining feasible items  
is identified, 
and as many units of it as possible are selected. 
This step is repeated until there are no feasible items left (i.e. at most $K$ rounds).

Another way to approximate the optimal solution of the unbounded knapsack problem is the \emph{fractional relaxation based} algorithm. 
This relaxes the original problem to its fractional version. 
In particular, within the \emph{fractional unbounded knapsack problem} we allow $x_{i}$ to be fractional. 
Now, it is easy to show that the optimal solution of the fractional unbounded knapsack is to solely choose \mbox{$I^{*} = \arg \max_{i}{\nicefrac{v_{i}}{w_{i}}}$} (i.e. $I^{*}$ is the item type with the highest density) \citep{KellererEtAl2004}.
That is, if $\mathbf{x^{*}} = \langle x^{*}_{1},\dots,x^{*}_{1}\rangle$ denotes the optimal solution of the fractional unbounded knapsack, then $x^{*}_{I^{*}} = \nicefrac{C}{w_{I^{*}}}$, while $\forall j \neq I^{*}$, $x_{j} = 0$.
Given this, within the original unbounded knapsack problem (where $x_{i}$ are integers), the fractional relaxation based algorithm chooses $x_{I^{*}} = \lfloor \nicefrac{C}{w_{I^{*}}} \rfloor$, and $x_{j} = 0$, $\forall j \neq I^{*}$. 
It can easily shown that the complexity of this algorithm is $O\left(K\right)$, which is the cost of determining the highest density type.  

\subsection{KUBE}
\label{Subsection:KUBE}

\noindent
The \alg\ algorithm is depicted in Algorithm~\ref{alg:online_learning_alg}.
Here, let $t$ denote the time step, and $B_{t}$ denote the residual budget at time $t\geq1$, respectively.
Note that at the start (i.e. $t = 1$), $B_{1} = B$, where $B$ is the total budget limit.    
At each subsequent time step, $t$, \alg\ first checks that arm pulling is feasible. 
That is, it is feasible only if at least one of the arms can be pulled with the remaining budget. 
Specifically, if $B_{t} < \min_{j}{c_j}$ (i.e. the residual budget is smaller than the lowest pulling cost), 
then \alg\ stops (steps $3-4$).

If arm pulling is still feasible, \alg\ first pulls each arm once in the initial phase (steps $6 - 7$).
Following this, at each time step $t > K$, it estimates the best set of arms according to their upper confidence bound 
using the density--ordered greedy approximation method 
applied to the following problem:
\vspace{-0.2\baselineskip}
\begin{eqnarray}
\label{eq:best_combination}
&&\max{\sum_{i=1}^{K}{m_{i,t}\left(\hat{\mu}_{i,n_{i,t}} + \sqrt{\frac{2\ln t}{n_{i,t}}}\, \right)}} \\
\nonumber
&&\quad \mathrm{s.t.} \quad \sum_{i=1}^{K}{m_{i,t}c_{i}} \leq B_{t}, \: \forall i,t: m_{i,t} \: \mathrm{integer}. 
\end{eqnarray}
\normalsize
In the above expression, 
$\hat{\mu}_{i,n_{i,t}}$ is the current estimate of arm $i$'s expected reward
(calculated as the average reward received so far from pulling arm $i$),
$n_{i,t}$ is the number of pulls of arm $i$ until time step $t$, 
and $\sqrt{\frac{2\ln t}{n_{i,t}}}$ is the size of the upper confidence interval.
The goal, then, is to find integers $\{m_{i,t}\}_{i\in K}$ 
such that Equation~\ref{eq:best_combination} is maximised, 
with respect to the residual budget limit $B_{t}$ 
(n.b.~from here on, we drop the subscript ${i\in K}$ on this set).
Since this problem is NP--hard, we use the density--ordered greedy method to find a near--optimal set of arms (step $9$).
Note that the upper confidence bound on arm $i$'s density is: 
\begin{equation}
\frac{\hat{\mu}_{i,n_{i,t}}}{c_{i}} + \frac{\sqrt{\frac{2\ln t}{n_{i,t}}}}{c_{i}}.
\end{equation}
\normalsize
Let $M^*(B_t) = \{m^{*}_{i,t}\}$ be this method's solution to the problem in Equation \ref{eq:best_combination}, 
giving us the desired set of arms, where $m^{*}_{i,t}$ is an index of arm $i$ that indicates how many times arm $i$ is taken into account within the set.
Using $\{m^{*}_{i,t}\}$, \alg\ \emph{randomly} chooses the next arm to pull, $i(t)$, by selecting arm $i$ with probability (step $10$): 
\begin{equation}
P\left(i\left(t\right) = i\right) 
 = \frac{m^{*}_{i,t}}{\sum_{k=1}^{K}{m^{*}_{k,t}}} \,.
\end{equation} 
\normalsize
After the pull, it then updates the estimated upper bound of the chosen arm, and the residual budget limit $B_t$ (steps $12 - 13$).

The intuition behind \alg\ is the following. 
By repeatedly drawing the next arm to pull from a distribution formed by the current estimated approximate best set, 
the expected reward of \alg\ equals the average reward for following the optimal solution 
to the corresponding unbounded knapsack problem, given the current reward estimates.
If the true values of the arms were known, then this would imply that 
the average performance of \alg\ efficiently converges to the optimal solution 
of the unbounded knapsack problem reduced from the budget--limited MAB model.
It is easy to show that the optimal solution of this knapsack model forms the theoretical optimal policy of the budget--limited MAB in case of having full information.
Put differently, if the mean reward value of each arm is known, then the budget--limited problem can be reduced to the unbounded knapsack problem, and thus, the optimal solution of the knapsack problem is the optimal solution of the budget--limited MAB as well.  
In addition, by combining the upper confidence bound with the estimated mean values of the arms, 
we guarantee that an arm that is not yet sampled many times may be pulled more frequently, 
since its upper confidence interval is large. 
Thus, we explore and exploit at the same time 
(for more details, see \citep{AuerEtal2002,AudibertEtAl2009}).
Note that, by using the density--ordered greedy method, 
\alg\ achieves a $O\left(K\ln{K}\right)$ computational cost per time step.

\small{
 \begin{algorithm}[t!]
 \caption{The \alg Algorithm}
 \label{alg:online_learning_alg}
{\fontsize{10}{10}\selectfont
\begin{algorithmic}[1] 
 \STATE \normalsize$t=1$; $B_{t}=B$; $\gamma > 0$; 
 \WHILE{pulling is feasible}
    \IF {$B_{t} < \min_{i}{c_i}$}
 	\STATE STOP! \{pulling is not feasible\}
    \ENDIF    
    \IF {$t \leq K$} 
	\STATE Initial phase: play arm $i\left(t\right) = t$;
    \ELSE   
         \STATE use density--ordered greedy to calculate $M^*(B_t) = \{m^{*}_{i,t}\}$, the solution of Equation \ref{eq:best_combination}; 
 	 \STATE randomly pull $i\left(t\right)$ with $P\left(i\left(t\right) = i\right) = \frac{m^{*}_{i,t}}{\sum_{k=1}^{K}{m^{*}_{k,t}}}$;
    \ENDIF  
    \STATE update the estimated upper bound of arm $i\left(t\right)$;
    \STATE $B_{t+1} = B_{t} - c_{i\left(t\right)}$; $t = t+1$;
 \ENDWHILE
\end{algorithmic}
}
\end{algorithm}
}
\normalsize

\vspace{-0.1cm}
\subsection{Fractional KUBE}
\label{Subsection:fractional_KUBE}
\noindent
We now turn to the fractional version of \alg, which follows the underlying concept of \alg. 
It also approximates the underlying unbounded knapsack problem at each time step $t$ in order to determine the frequency of arms within the estimated best set of arms.
However, it differs from \alg\ by using the fractional relaxation based method to approximate the unbounded knapsack in Step $9$ of Algorithm~\ref{alg:online_learning_alg}.
Crucially, fractional \alg\ uses the fractional relaxation based algorithm to solve the following fractional unbounded knapsack problem at each $t$:
\vspace{-0.2\baselineskip}
\begin{equation}
\label{eq:best_fractional_combination}
\max{\sum_{i=1}^{K}{m_{i,t}\left(\hat{\mu}_{i,n_{i,t}} + \sqrt{\frac{2\ln t}{n_{i,t}}}\, \right)}} 
 \quad \mathrm{s.t.} \quad \sum_{i=1}^{K}{m_{i,t}c_{i}} \leq B_{t}. 
\end{equation}
\normalsize
Recall that within \alg, the frequency of arms within the approximated solution of the unbounded knapsack forms a probability distribution from which the agent randomly pulls the next arm.
Now, since the fractional relaxation based algorithm solely chooses the arm (i.e. item type) with the highest estimated confidence bound--cost ratio (i.e. item density), fractional \alg\ does not need to randomly choose an arm.
Instead, at each time step $t$, it pulls the arm that maximises $\left(\nicefrac{\hat{\mu}_{i,n_{i,t}}}{c_{i}} + \nicefrac{\sqrt{\frac{2\ln t}{n_{i,t}}}}{c_{i}}\right)$.
That is, fractional \alg\ can also be seen as the budget--limited version of UCB (see \citep{AuerEtal2002} for more details of UCB).

Computation--wise, by replacing the density--ordered greedy with the fractional relaxation based algorithm, fractional \alg\ decreases the computational cost to $O\left(K\right)$ per time step.
In what follows, we show that both \alg\ and its fractional counterpart achieve asymptotically optimal regret bounds.

\section{Performance Analysis}
\label{Section:perf_analysis}

\noindent 
We now focus on the analysis of the expected regret of \alg\ and fractional \alg,
defined by Equation~\ref{eq:loss_function}. 
To this end, in this section we: (i) derive an upper bound on the regret of the algorithms, 
and (ii) show that these bounds are asymptotically optimal.

To begin, let us state some simplifying assumptions and define some useful terms. 
Without loss of generality, for ease of exposition we assume 
that the reward distribution of each arm has support in $\left[0,1\right]$, 
and that the pulling cost $c_i \geq 1$ for each $i$ (our result can be scaled for different size supports and costs as appropriate). 
Let \mbox{$I^{*} = \arg \max_{i}{\nicefrac{\mu_{i}}{c_{i}}}$} be the arm with the highest true mean value density.
For the sake of simplicity, we assume that $I^{*}$ is unique (however, our proofs do not exploit this fact at all).
Let $d_{\mathrm{min}} = \min_{j \neq I^{*}}{\{\nicefrac{\mu_{I^{*}}}{c_{I^{*}}} - \nicefrac{\mu_{j}}{c_{j}}\}}$ 
denote the minimal true mean value density difference of arm $I^{*}$ and that of any other arm $j$.
In addition, let $c_{\mathrm{min}} = \min_{j}{c_j}$ and $c_{\mathrm{max}} = \max_{j}{c_j}$ 
denote the smallest and largest pulling costs, respectively.
Then let $\delta_j = c_j - c_{I^{*}}$ 
be the difference of arm $j$'s pulling cost and the minimal pulling cost.
Similarly, let \mbox{$\Delta_j = \mu_{I^{*}} - \mu_{j}$} 
denote the difference of the highest true mean value and that of arm $j$. 
Note that both $\delta_j$ and $\Delta_j$ could be negative values, 
since $I^{*}$ does not necessarily have the highest true mean value, 
nor the smallest pulling cost. 
In addition, let $T$ denote the finite--time operating time of the agent.

Now, we first analyse the performance of \alg. 
In what follows, we first estimate the number of times we pull arm $j \neq I^{*}$, instead of $I^{*}$.
Based on this result, we estimate $\E\left[T\right]$, the average number of pulls of \alg.
This bound guarantees that \alg\ always pulls ``enough'' arms 
so that the difference of the number of pulls in the theoretical optimal solution and that of \alg\ is small, 
compared to the size of the budget.
By using the estimated value of $\E\left[T\right]$, we then show that \alg\ achieves a $O\left(\ln{\left(B\right)}\right)$ worst case regret on average.
In more detail, we get:
\begin{theorem}[Main result 1]
\label{theorem:KUBE_regret_upper_bound_KUBE}
For any budget size $B > 0$, the performance regret of \algKUBE\ is at most:
\small
\begin{eqnarray*}
\label{eq:KUBE_regret_upper_bound_KUBE}
\left(\frac{8}{d^2_{\mathrm{min}}} + \left(\frac{\cmax}{\cmin}\right)^2\right)\left(\sum_{\Delta_j > 0}{\Delta_j} 
 + \sum_{\delta_j > 0}{\frac{\delta_j}{c_{I^{*}}}}\right)\ln{\left(\frac{B}{\cmin}\right)}
 + \left(\sum_{\Delta_j > 0}{\Delta_j} 
 + \sum_{\delta_j > 0}{\frac{\delta_j}{c_{I^{*}}}}\right)\left( \frac{\pi^2}{3} + 1 \right)+1\,.
\end{eqnarray*}
\normalsize
\end{theorem}
It is easy to show that for each $j \neq I^{*}$, at least one between $\delta_j$ and $\Delta_j$ has to be positive. 
This implies that $\left(\sum_{\Delta_j > 0}{\Delta_j} + \sum_{\delta_j > 0}{\frac{\delta_j}{c_{I^{*}}}}\right) > 0$.
That is, the performance regret of \alg\ (i.e. $R\left(\alg\right)$) is upper--bounded by $O\left(\ln{B}\right)$.
To prove this theorem, we will make use of the following version of the Chernoff--Hoeffding concentration inequality for bounded random variables:
\begin{theorem}[Chernoff--Hoeffding inequality \citep{Hoeffding1963}]
\label{theorem:chernoff_hoeffding}
Let $X_1$, $X_2$, $\dots, X_n$ denote the sequence of random variables with common range $\left[0,1\right]$, such that for any $1 \leq t \leq n$, we have $\E\left[X_t| X_1,\dots,X_{t-1}\right] = \mu$.
Let $S_n = \frac{1}{n}\sum_{t=1}^{n}{X_t}$.  
Given this, for any $\delta \geq 0$, we have:
\begin{eqnarray}
\nonumber
P\left(S_n \geq \mu + \delta \right) \leq e^{-2n\delta^2}, \\
\nonumber
P\left(S_n \leq \mu - \delta \right) \leq e^{-2n\delta^2}.
\end{eqnarray}
\end{theorem}
The proof can be found, for example, in \citeauthor{Hoeffding1963} (\citeyear{Hoeffding1963}).

We now focus on the performance analysis of \algKUBE. 
To this end, we introduce some further notation.
Let $T$ denote the number of pulls of \algKUBE. 
In addition, let $N_{j}\left(T\right)$ denote the number of times \algKUBE\ pulls arm $j$ up to time step $T$.

In what follows, we first devise an upper bound for  $N_{j}\left(T\right)$ for all $j \neq I^{*}$.
That is, we estimate the number of times we pull arm $j \neq I^{*}$, instead of $I^{*}$.
Based on this result, we estimate the average number of pulls of \algKUBE\ (i.e. $\E\left[T\right]$).
This bound guarantees that \algKUBE\ always pulls ``enough'' arms 
so that the difference between the number of pulls in the theoretical optimal solution and that of \algKUBE\ is small, 
compared to the size of the budget.
By using the estimated value of $\E\left[T\right]$, we then show that \algKUBE\ achieves a $O\left(\ln{\left(B\right)}\right)$ worst case regret on average.
%
We now state the following:
\begin{lemma}
\label{lemma:KUBE_pull_number_j_KUBE}
Suppose that \algKUBE\ pulls the arms $T$ times. If $j \neq I^{*}$, then:
\begin{equation*}
\E\left[N_{j}\left(T\right)| T\right] \leq \left(\frac{8}{d^2_{\mathrm{min}}} + \left( \frac{\cmax}{\cmin}\right)^2 \right)\ln{\left(T\right)} + \frac{\pi^2}{3} + 1.
\end{equation*}
\end{lemma}
That is, the number of times \algKUBE\ pulls an arm $j \neq I^{*}$ is at most $O\left(\ln{\left(T\right)}\right)$.
To prove this lemma, let us first refresh some of the terms that are used:
$i\left(t\right)$ is the arm pulled by \algKUBE\ at time $t$; 
when refering to a set of arms $\{m_{j,t}\}$, $m_{j,t}$ is the number of pulls of arm $j$;
$M^*(B_t) = \{m^{*}_{i,t}\}$ is the density--ordered greedy approximate solution to unbounded knapsack problem 
in Equation $6$, 
where $m^{*}_{i,t}$ is the number of arm $i$'s pulls in this set; 
and $I^{*} = \arg \max_{i}{\frac{\mu_{i}}{c_{i}}}$ is the arm with the highest true mean value density.
In addition, $\hat{I}\left(t\right) = \arg\max_{j}{\left\{\frac{\hat{\mu}_{j,n_{j,t}}}{c_{j}} + \frac{\sqrt{\frac{2\ln t}{n_{j,t}}}}{c_{j}}\right\}}$ is the arm with the highest estimated density confidence bound at time step $t$.
In order to prove Lemma~\ref{lemma:KUBE_pull_number_j_KUBE}, we rely on the following lemmas:
\begin{lemma}
\label{lemma:KUBE_lemma5}
Suppose that the total number of pulls \algKUBE\ makes of the arms is $T$,
and that at each time step $t$, the residual budget is $B_{t}$ (note that here $B_{1} = B$). 
For any $0 < t \leq T$, we have: 
\begin{equation*}
\frac{\cmin}{B_{t}} \leq \frac{1}{T-t+1} .
\end{equation*}
\end{lemma}

\begin{lemma}
\label{lemma:KUBE_lemma6}
Suppose that the total number of pulls \algKUBE\ makes of the arms is $T$. 
For any $0 < t \leq T$, we have: 
\begin{equation*}
P\left(i\left(t\right) = j\right | T) \leq P\left(\hat{I}\left(t\right) = j | T\right) + \left(\frac{\cmax}{\cmin}\right)^2\frac{1}{T-t+1} .
\end{equation*}
\end{lemma}

\begin{proof}{of Lemma~\ref{lemma:KUBE_lemma5}}
At the beginning of time step $t$, the residual budget is $B_t$. 
Since the total number of pulls is $T$, with respect to $B_t$, 
\algKUBE\ can still make $T-t+1$ pulls (including the pull at time step $t$).
This indicates that:
\begin{equation*}
\label{eq:KUBE_lemma5_1}
B_{t} \geq c_{i\left(t\right)} + c_{i\left(t+1\right)} 
 + \dots + c_{i\left(T\right)} \geq \left(T-t+1\right)\cmin .
\end{equation*}
which directly implies the inequality in Lemma~\ref{lemma:KUBE_lemma5}.
\end{proof}

\begin{proof}{of Lemma~\ref{lemma:KUBE_lemma6}}
We assume that the value of $T$ is given.
For the slight abuse of notation, we drop the conditional of $T$ notation to simplify the proof (i.e. all the probabilities are considered to be conditional to $T$), and we will explicitly denote it when necessary.  
First, we consider a particular value of $B_t$. 
Thus, we have:
\begin{equation}
\label{eq:KUBE_lemma6_1}
P\left(i\left(t\right) = j | B_{t}\right) = \sum_{\left\{m_{i,t}\right\}}{P\left(i\left(t\right) = j | M^{*}\left(B_{t}\right) = \left\{m_{i,t}\right\} \right)P\left(M^{*}\left(B_{t}\right) = \left\{m_{i,t}\right\}\right)}.
\end{equation}
Recall that the density--ordered greedy approach first repeatedly adds arm $\hat{I}\left(t\right)$ to set $\left\{m_{i,t}\right\}$ until it is not feasible.
It is easy to show that after adding arm $\hat{I}\left(t\right)$ as many times as possible 
(i.e. $m_{\hat{I}\left(t\right),t}$ times) to the set, the residual budget is at most $c_{\hat{I}\left(t\right)}$ 
(or otherwise we could still add arm $\hat{I}\left(t\right)$ one more time).
Therefore:
\begin{equation}
\label{eq:KUBE_lemma6_2}
\sum_{i \neq \hat{I}\left(t\right)}{m_{i,t}} \leq \frac{c_{\hat{I}\left(t\right)}}{\cmin}.
\end{equation}
That is, the total count of arm pulls other than $\hat{I}\left(t\right)$ in the set is at most $\frac{c_{\hat{I}\left(t\right)}}{\cmin}$. 
This inequality comes from the fact that we can construct a set with the greatest number of arm pulls 
by only adding the arm with the smallest cost.
Similarly, we have:
\begin{equation}
\label{eq:KUBE_lemma6_3}
\sum_{k=1}^{K}{m_{k,t}} \geq \frac{B_{t}}{\cmax},
\end{equation}
because we can construct a set with the smallest number of arm pulls by only adding the arm with the greatest cost.
Combining Equations~\ref{eq:KUBE_lemma6_2} and~\ref{eq:KUBE_lemma6_3} gives:
\begin{equation}
\label{eq:KUBE_lemma6_4}
\frac{ \sum_{i \neq \hat{I}\left(t\right)}{m_{i,t}} }{\sum_{k=1}^{K}{m_{k,t}}} \leq \frac{\frac{c_{\hat{I}\left(t\right)}}{\cmin}}{ \frac{B_{t}}{\cmax} } \leq \left(\frac{\cmax}{\cmin}\right)^2\frac{\cmin}{B_{t}}.
\end{equation}
The last inequality is obtained from the fact that $c_{\hat{I}\left(t\right)} \leq \cmax$.
Now, recall that \algKUBE\ chooses arm $j$ to pull with probability $\frac{m_{j,t}}{\sum_{k=1}^{K}{m_{k,t}}}$.
This implies that:
\begin{flalign}
\nonumber
P\big(i\left(t\right) = j & | M^{*}\left(B_{t}\right) = \left\{m_{i,t}\right\} \big)  \\
\nonumber
&= P\left(i\left(t\right) = j, \hat{I}\left(t\right) = j | M^{*}\left(B_{t}\right) = \left\{m_{i,t}\right\} \right) \\
\nonumber
&+ P\left(i\left(t\right) = j, \hat{I}\left(t\right) \neq j | M^{*}\left(B_{t}\right) = \left\{m_{i,t}\right\} \right).
\end{flalign}
\normalsize
This can be upper bounded by:
\begin{flalign}
\nonumber
P\big(i\left(t\right) = j & | M^{*}\left(B_{t}\right) = \left\{m_{i,t}\right\} \big)  \\
&\leq \frac{m_{\hat{I}\left(t\right),t}}{\sum_{k=1}^{K}{m_{k,t}}}P\left(\hat{I}\left(t\right) = j | M^{*}\left(B_{t}\right) = \left\{m_{i,t}\right\} \right) \\
\nonumber
&+ \frac{\sum_{i \neq \hat{I}\left(t\right)}{m_{i,t}}}{\sum_{k=1}^{K}{m_{k,t}}}P\left(\hat{I}\left(t\right) \neq j | M^{*}\left(B_{t}\right) = \left\{m_{i,t}\right\} \right).
\end{flalign}
The right hand side can be further upper bounded as follows:
\begin{flalign}
\label{eq:KUBE_lemma6_5}
\nonumber
P\big(i\left(t\right) = j & | M^{*}\left(B_{t}\right) = \left\{m_{i,t}\right\} \big)  \\
\nonumber
&\leq P\left(\hat{I}\left(t\right) = j | M^{*}\left(B_{t}\right) = \left\{m_{i,t}\right\} \right) + \frac{\sum_{i \neq \hat{I}\left(t\right)}{m_{i,t}}}{\sum_{k=1}^{K}{m_{k,t}}} \\
&\leq P\left(\hat{I}\left(t\right) = j | M^{*}\left(B_{t}\right) = \left\{m_{i,t}\right\} \right) + \left(\frac{\cmax}{\cmin}\right)^2\frac{\cmin}{B_{t}}.
\end{flalign}
The last inequality is obtained from Equation~\ref{eq:KUBE_lemma6_4}.
Substituting Equation~\ref{eq:KUBE_lemma6_5} into Equation~\ref{eq:KUBE_lemma6_1} gives:
\small
\begin{flalign}
\label{eq:KUBE_lemma6_6}
\nonumber
P\left(i\left(t\right) = j | B_{t}\right) &\leq \sum_{\left\{m_{i,t}\right\}}{\left( P\left(\hat{I}\left(t\right) = j | M^{*}\left(B_{t}\right) = \left\{m_{i,t}\right\} \right) + \left(\frac{\cmax}{\cmin}\right)^2\frac{\cmin}{B_{t}} \right)P\left(M^{*}\left(B_{t}\right) = \left\{m_{i,t}\right\}\right)} \\
\nonumber
&\leq P\left(\hat{I}\left(t\right) = j | B_{t}\right) + \left(\frac{\cmax}{\cmin}\right)^2\frac{\cmin}{B_{t}} \\
&\leq P\left(\hat{I}\left(t\right) = j | B_{t}\right) + \left(\frac{\cmax}{\cmin}\right)^2\frac{1}{T - t + 1}.
\end{flalign}
\normalsize
The last inequality is obtained from Lemma~\ref{lemma:KUBE_lemma5}.
Now we study the general case, where $B_t$ is not fixed.
By summing up Equation~\ref{eq:KUBE_lemma6_6} over all possible value of $B_t$, we have:
\small
\begin{flalign}
\label{eq:KUBE_lemma6_7}
\nonumber
P\left(i\left(t\right) = j | T\right) &= \sum_{B_t}{P\left(i\left(t\right) = j | T, B_t\right)P\left(B_t | T\right) } \\
\nonumber
&\leq \sum_{B_t}{\left(P\left(\hat{I}\left(t\right) = j | T, B_t\right) + \left(\frac{\cmax}{\cmin}\right)^2\frac{1}{T - t + 1}\right)P\left(B_t | T\right)} \\
&\leq P\left(\hat{I}\left(t\right) = j | T \right) + \left(\frac{\cmax}{\cmin}\right)^2\frac{1}{T - t + 1}. 
\end{flalign}
\normalsize
which concludes the proof.
\end{proof}

Based on Lemmas~\ref{lemma:KUBE_lemma5} and~\ref{lemma:KUBE_lemma6}, Lemma~\ref{lemma:KUBE_pull_number_j_KUBE} can be proved as follows:
\begin{proof}{of Lemma~\ref{lemma:KUBE_pull_number_j_KUBE}}
We assume that the value of $T$ is already given.
Again, for the slight abuse of notation, we drop the conditional of $T$ notation to simplify the proof, and we will explicitly denote it when necessary.
In this case, the proof of the theorem for that particular value of $T$ is along the same lines as that of Theorem 1 of \citep{AuerEtal2002}.
In particular, recall that $N_{j}\left(T\right)$ denotes the expectation of number of times \algKUBE\ pulls an arm $j \neq I^{*}$ until time step $T$.
Given this, we have the following:
\begin{flalign}
\label{eq:KUBE_theorem1_1}
\nonumber
\E\left[N_{j}\left(T\right)\right] &= 1 + \sum_{t=K+1}^{T}{P\left(i\left(t\right) = j\right)} \\
\nonumber
&\leq  1 + \sum_{t=K+1}^{T}{P\left(\hat{I}\left(t\right) = j\right)} +  \sum_{t=K+1}^{T}{\left(\frac{\cmax}{\cmin}\right)^2\frac{1}{T - t + 1}} \\
&\leq  l + \sum_{t=K+1}^{T}{P\left(\hat{I}\left(t\right) = j, N_{j}\left(t\right) \geq l\right)} +  \sum_{t=K+1}^{T}{\left(\frac{\cmax}{\cmin}\right)^2\frac{1}{T - t + 1}}
\end{flalign}
for any $l \geq 1$.
Now, let $b_{t,s} = \sqrt{\frac{2\ln{t}}{s}}$.
Considering the second term on the right hand side of Equation~\ref{eq:KUBE_theorem1_1}, we have:
\small
\begin{flalign}
\label{eq:KUBE_theorem1_2}
\nonumber
\sum_{t=K+1}^{T}{P\left(\hat{I}\left(t\right) = j, N_{j}\left(t\right) \geq l\right)} &= \sum_{t=K+1}^{T}{P\left(\frac{\hat{\mu}_{I^{*},N_{I^{*}}\left(t\right)}}{c_{I^{*}}} + \frac{b_{t,N_{I^{*}}\left(t\right)}}{c_{I^{*}}} \leq \frac{\hat{\mu}_{j,N_{j}\left(t\right)}}{c_{j}} + \frac{b_{t,N_{j}\left(t\right)}}{c_{j}}, N_{j}\left(t\right) \geq l\right)} \\
\nonumber
&\leq \sum_{t=K+1}^{T}{P\left(\min_{1 \leq s \leq t}{\left\{ \frac{\hat{\mu}_{I^{*},s}}{c_{I^{*}}} + \frac{b_{t,s}}{c_{I^{*}}}\right\}} \leq \max_{l \leq s_j \leq t}{\left\{\frac{\hat{\mu}_{j,s_j}}{c_{j}} + \frac{b_{t,s_j}}{c_{j}}\right\}}\right)} \\
&\leq \sum_{t=1}^{T}{\sum_{s=1}^{t}{\sum_{s_j = 1}^{t}{P\left(\frac{\hat{\mu}_{I^{*},s}}{c_{I^{*}}} + \frac{b_{t,s}}{c_{I^{*}}} \leq \frac{\hat{\mu}_{j,s_j}}{c_{j}} + \frac{b_{t,s_j}}{c_{j}} \right)} }}.
\end{flalign}
\normalsize
If it is true that $\frac{\hat{\mu}_{I^{*},s}}{c_{I^{*}}} + \frac{b_{t,s}}{c_{I^{*}}} \leq \frac{\hat{\mu}_{j,s_j}}{c_{j}} + \frac{b_{t,s_j}}{c_{j}}$, 
then at least one of the following three statements must also hold:
\begin{equation}
 \label{eq:KUBE_theorem1_3}
 \frac{\hat{\mu}_{I^{*},s}}{c_{I^{*}}} + \frac{b_{t,s}}{c_{I^{*}}} \leq \frac{\mu_{I^{*}}}{c_{I^{*}}},
\end{equation}
\begin{equation}
 \label{eq:KUBE_theorem1_4} 
 \frac{\mu_{j}}{c_{j}} \leq \frac{\hat{\mu}_{j,s_j}}{c_{j}} + \frac{b_{t,s_j}}{c_{j}},
\end{equation}
\begin{equation}
 \label{eq:KUBE_theorem1_5}
 \frac{\mu_{I^{*}}}{c_{I^{*}}} \leq \frac{\mu_{j}}{c_{j}} + \frac{2b_{t,s_j}}{c_{j}}.
\end{equation}
That is, we get:
\begin{flalign}
\label{eq:KUBE_theorem1_6}
\nonumber
P\left(\frac{\hat{\mu}_{I^{*},s}}{c_{I^{*}}} + \frac{b_{t,s}}{c_{I^{*}}} \leq \frac{\hat{\mu}_{j,s_j}}{c_{j}} + \frac{b_{t,s_j}}{c_{j}} \right) 
\leq& P\left( \frac{\hat{\mu}_{I^{*},s}}{c_{I^{*}}} + \frac{b_{t,s}}{c_{I^{*}}} \leq \frac{\mu_{I^{*}}}{c_{I^{*}}} \right) + \\
  & + P\left(\frac{\mu_{j}}{c_{j}} \leq \frac{\hat{\mu}_{j,s_j}}{c_{j}} + \frac{b_{t,s_j}}{c_{j}}  \right) 
    + P\left(\frac{\mu_{I^{*}}}{c_{I^{*}}} \leq \frac{\mu_{j}}{c_{j}} + \frac{2b_{t,s_j}}{c_{j}}  \right).
\end{flalign}
Applying the Chernoff--Hoeffding inequalities to the first two terms on the right hand side of Equation~\ref{eq:KUBE_theorem1_6} gives:
\begin{flalign}
\label{eq:KUBE_theorem1_7}
P\left( \frac{\hat{\mu}_{I^{*},s}}{c_{I^{*}}} + \frac{b_{t,s}}{c_{I^{*}}} \leq \frac{\mu_{I^{*}}}{c_{I^{*}}} \right) &= P\left(\hat{\mu}_{I^{*},s} + b_{t,s} \leq \mu_{I^{*}} \right) \leq \exp{\left\{ -2b^{2}_{t,s}s\right\}} = \exp{\left\{ -4\ln{t}\right\}} = t^{-4} \\
\label{eq:KUBE_theorem1_8}
P\left(\frac{\mu_{j}}{c_{j}} \leq \frac{\hat{\mu}_{j,s_j}}{c_{j}} + \frac{b_{t,s_j}}{c_{j}} \right) &= P\left(\mu_{j} \leq \hat{\mu}_{j,s_j} + b_{t,s_j}\right) \leq \exp{\left\{ -2b^{2}_{t,s_j}s_j\right\}} = \exp{\left\{ -4\ln{t}\right\}} = t^{-4}.
\end{flalign}
On the other hand, for $l \geq \frac{8\ln{T}}{d^{2}_{\mathrm{min}}}$, Equation~\ref{eq:KUBE_theorem1_5} is false, since:  
\begin{flalign}
\label{eq:KUBE_theorem1_9}
\nonumber
\frac{\mu_{I^{*}}}{c_{I^{*}}} - \frac{\mu_{j}}{c_{j}} - \frac{2b_{t,s_j}}{c_{j}} &\geq \frac{\mu_{I^{*}}}{c_{I^{*}}} - \frac{\mu_{j}}{c_{j}} - 2b_{t,s_j} \\
\nonumber
&\geq \frac{\mu_{I^{*}}}{c_{I^{*}}} - \frac{\mu_{j}}{c_{j}} - 2\sqrt{\frac{2\ln{t}}{l}}\\
\nonumber
&\geq \frac{\mu_{I^{*}}}{c_{I^{*}}} - \frac{\mu_{j}}{c_{j}} - 2\sqrt{\frac{2\ln{t}}{ \frac{8\ln{T}}{d^{2}_{\mathrm{min}}} }} \\
\nonumber
&\geq \frac{\mu_{I^{*}}}{c_{I^{*}}} - \frac{\mu_{j}}{c_{j}} - d_{\mathrm{min}}\\
&\geq \frac{\mu_{I^{*}}}{c_{I^{*}}} - \frac{\mu_{j}}{c_{j}} - d_{j} = 0. 
\end{flalign}
Here note that $c_{j} \geq 1$, $s_j \geq l \geq \frac{8\ln{T}}{d^{2}_{\mathrm{min}}}$, and $t \leq T$.
If $l \geq \frac{8\ln{T}}{d^{2}_{\mathrm{min}}}$, then $P\left(\frac{\mu_{I^{*}}}{c_{I^{*}}} \leq \frac{\mu_{j}}{c_{j}} + \frac{2b_{t,s_j}}{c_{j}}\right) = 0$.
Substituting this and Equations~\ref{eq:KUBE_theorem1_6},~\ref{eq:KUBE_theorem1_7} and~\ref{eq:KUBE_theorem1_8} into Equation~\ref{eq:KUBE_theorem1_2} gives:
\begin{equation}
\label{eq:KUBE_theorem1_10}
\sum_{t=K+1}^{T}{P\left(\hat{I}\left(t\right) = j, N_{j}\left(t\right) \geq l\right)} 
\leq \sum_{t=1}^{T}{\sum_{s=1}^{t}{\sum_{s_j = 1}^{t}{2t^{-4}} }} \leq \frac{\pi^2}{3},
\end{equation}
for any $l \geq \left\lceil \frac{8\ln{T}}{d^{2}_{\mathrm{min}}} \right\rceil$.
Note that the last inequality is obtained from the Riemann Zeta Function for value of $2$ (i.e. $\sum_{t=1}^{\infty}{t^{-2}} = \frac{\pi^2}{6}$) \citep{Ivic1985}.

Now, consider the third term on the right hand side of Equation~\ref{eq:KUBE_theorem1_1}.
By using Lemma~\ref{lemma:KUBE_lemma5}, we get:
\begin{equation}
\label{eq:KUBE_theorem1_11}
\sum_{t=1}^{T}{\left(\frac{\cmax}{\cmin}\right)^2\frac{1}{T-t+1}} \leq \left(\frac{\cmax}{\cmin}\right)^2\ln{\left(T\right)}.
\end{equation}
We now combine Equations~\ref{eq:KUBE_theorem1_10} and~\ref{eq:KUBE_theorem1_11} together, 
and we set $l = \frac{8\ln{T}}{d^{2}_{\mathrm{min}}} + 1$, which gives:
\begin{flalign}
\label{eq:KUBE_theorem1_12}
\nonumber
\E\left[N_{j}\left(T\right)\right] \leq 
\frac{8\ln{T}}{d^{2}_{\mathrm{min}}} + 1 + \frac{\pi^2}{3} + \left(\frac{\cmax}{\cmin}\right)^2\ln{\left(T\right)}
\end{flalign}
for any given value of $T$, which concludes the proof.
\end{proof}
From Lemma~\ref{lemma:KUBE_pull_number_j_KUBE}, we can show the following:
\begin{lemma}
\label{lemma:KUBE_pulls_small_budget_KUBE}
Suppose that the total budget size is $B$.
If $T$ denotes the total number of pulls of \algKUBE\, then we have:
\begin{equation*}
\label{eq:KUBE_pulls_lowerbound_small_budget_KUBE}
\E\left[T\right] \geq \frac{B}{c_{I^{*}}} 
 - \left(\frac{8}{d^2_{\mathrm{min}}} + \left(\frac{\cmax}{\cmin}\right)^2\right)\sum_{\delta_j > 0}{\frac{\delta_j}{c_{I^{*}}}} \ln{\left(\frac{B}{\cmin}\right)} 
 - \sum_{\delta_j > 0}{\frac{\delta_j}{c_{I^{*}}}}\left( \frac{\pi^2}{3} + 1 \right) - 1 
\end{equation*}
\normalsize
where $\E\left[T\right]$ is the expected number of pulls using \algKUBE.
\end{lemma}
That is, the difference between $\frac{B}{c_{I^{*}}}$ 
and the number of pulls of \algKUBE\ is at most 
$O\left(\ln{\left(\frac{B}{\cmin}\right)}\right)$.
%
%
\begin{proof}{of Lemma~\ref{lemma:KUBE_pulls_small_budget_KUBE}}
Since \algKUBE\ pulls arms until none are feasible, by definition:
\begin{equation*}
\label{eq:KUBE_theorem2_1}
P\left( \sum_{t=1}^{T}{c_{i\left(t\right)}} \leq B - \cmin\right) = 1 .
\end{equation*}
Taking the expectation of $\sum_{t=1}^{T}{c_{i\left(t\right)}}$ over $T$ and $\{m_{j,t}\}$ 
(i.e. the set of $i\left(t\right)$) gives:
\begin{flalign}
\nonumber
B - \cmin &\leq \E_{T, \{i\left(t\right)\}}\left[\sum_{t=1}^{T}{c_{i\left(t\right)}}\right] 
 = \E_{T}\left[\sum_{t=1}^{T}{\E_{i\left(t\right)}\left[c_{i\left(t\right)}\right]}\right] \\
\nonumber
&\leq \E_{T}\left[ \sum_{t=1}^{T}{ \sum_{j=1}^{K}{c_{j}P\left(i\left(t\right) = j | T\right)}}\right] \\
\nonumber
&\leq \E_{T}\left[ \sum_{t=1}^{T}{\left( c_{I^{*}} + \sum_{\delta_j > 0}{\delta_{j}P\left(i\left(t\right) = j | T\right)}\right)}\right] 
\end{flalign}
\begin{flalign}
\nonumber
&\leq \E_{T}\left[T\right]c_{I^{*}} 
  + \E_{T}\left[ \sum_{\delta_{j} > 0}{\delta_{j}\left( \sum_{t=1}^{T}{P\left(i\left(t\right) = j | T\right)}\right)}\right]\\ 
\label{eq:KUBE_theorem2_2}
&\leq \E_{T}\left[T\right]c_{I^{*}} 
  + \E_{T}\left[ \sum_{\delta_{j} > 0}{\delta_{j}\left(  \left(\frac{8}{d^2_{\mathrm{min}}} + \left( \frac{\cmax}{\cmin}\right)^2 \right)\ln{\left(T\right)} + \frac{\pi^2}{3} + 1  \right)}\right] \\
\label{eq:KUBE_theorem2_3}
&\leq \E_{T}\left[T\right]c_{I^{*}} 
  + \sum_{\delta_{j} > 0}{\delta_{j}\left( \left(\frac{8}{d^2_{\mathrm{min}}} + \left( \frac{\cmax}{\cmin}\right)^2 \right)\ln{\left(\frac{B}{\cmin}\right)} + \frac{\pi^2}{3} + 1  \right)} \,.
\end{flalign}
Equation~\ref{eq:KUBE_theorem2_2} is obtained from Lemma~\ref{lemma:KUBE_pull_number_j_KUBE}, while Equation~\ref{eq:KUBE_theorem2_3} comes from the fact that $T \leq \frac{B}{\cmin}$ with probability $1$.
In addition, the third inequality is obtained from the fact that $\delta_j$ can be smaller than $0$ for some $j$, and thus, we can further upper bound by only summing up $\delta_{j}P\left(i\left(t\right) = j | T\right)$ over arms that have $\delta_j > 0$. 
Now, by dividing both sides with $c_{I^{*}}$, we obtain:
\begin{equation*}
\frac{B}{c_{I^{*}}} - \frac{\cmin}{c_{I^{*}}} 
- \sum_{\delta_{j} > 0}{\frac{\delta_{j}}{c_{I^{*}}}\left(  \left(\frac{8}{d^2_{\mathrm{min}}} + \left( \frac{\cmax}{\cmin}\right)^2 \right)\ln{\left(\frac{B}{\cmin}\right)} + \frac{\pi^2}{3} + 1  \right)} 
\leq  \E_{T}\left[T\right].
\end{equation*}
By using the fact that $\frac{\cmin}{c_{I^{*}}} \leq 1$, we obtain the stated formula. 
\end{proof}
Note that if we relax the budget--limited MAB problem so that the number of pulls can be fractional, 
then it is easy to show that the optimal pulling policy of this relaxed model is to repeatedly pull arm $I^{*}$ only.
In this case, $\frac{B}{c_{I^{*}}}$ is the number of pulls of this optimal policy. 
Lemma~\ref{lemma:KUBE_pulls_small_budget_KUBE} indicates that the number of pulls that \algKUBE\ produces 
does not significantly differ from that of the optimal policy of the fractional budget--limited MAB 
(i.e. the difference is a logarithmic function of the number of pulls). 
We can now  derive the regret bound of \algKUBE\ from Lemma~\ref{lemma:KUBE_pulls_small_budget_KUBE} as follows:
\begin{proof}{of Theorem~\ref{theorem:KUBE_regret_upper_bound_KUBE}}
Recall that $\E\left[G^{B}\left(A^{*}\right)\right]$ denotes the expected performance of the theoretical optimal policy.
It is obvious that $\E\left[G^{B}\left(A^{*}\right)\right] \leq \frac{B\mu_{I^{*}}}{c_{I^{*}}}$, 
since the latter is the optimal solution of the fractional budget--limited MAB problem.
This indicates that:
\begin{flalign}
\label{eq:KUBE_theorem3_1}
\nonumber
R^{B}\left(\algKUBE\right) 
&= \E\left[G^{B}\left(A^{*}\right)\right] - \E\left[G^{B}\left(\algKUBE\right)\right] \\
\nonumber
&\leq \frac{B\mu_{I^{*}}}{c_{I^{*}}} 
 - \E_{T,\{i\left(t\right)\}}\left[ \sum_{t=1}^{T}{\mu_{i\left(t\right)}} \right] \\ 
\nonumber
&\leq \frac{B\mu_{I^{*}}}{c_{I^{*}}} 
 - \E_{T}\left[ \sum_{t=1}^{T}{ \E_{i\left(t\right)}\left[\mu_{i\left(t\right)}\right] } \right] \\
\nonumber
&\leq \E_{T}\left[ \frac{B\mu_{I^{*}}}{c_{I^{*}}} 
 - \sum_{t=1}^{T}{ \E_{i\left(t\right)}\left[\mu_{i\left(t\right)}\right] } \right] \\
\nonumber
&\leq \E_{T}\left[ \frac{B\mu_{I^{*}}}{c_{I^{*}}} 
 - \sum_{t=1}^{T}{\sum_{j}^{K}{\mu_{j}P\left(i\left(t\right)=j | T\right)}  } \right] \\
\nonumber
&\leq \E_{T}\left[ \left(\frac{B}{c_{I^{*}}} - T\right)\mu_{I^{*}}
 + \sum_{t=1}^{T}{\left( \mu_{I^{*}} - \sum_{j}^{K}{\mu_{j}P\left(i\left(t\right)=j | T\right)} \right) } \right] \\
\nonumber
&\leq \E_{T}\left[\frac{B}{c_{I^{*}}} - T\right]\mu_{I^{*}}
 + \E_{T}\left[\sum_{t=1}^{T}{\sum_{\Delta_j > 0}{\Delta_{j}P\left(i\left(t\right)=j | T\right)}} \right] \\
&\leq \E_{T}\left[\frac{B}{c_{I^{*}}} - T\right]\mu_{I^{*}}
 + \E_{T}\left[\sum_{\Delta_j > 0}{\Delta_{j}\E\left[N_{j}\left(T\right) | T\right]} \right].
\end{flalign} 
\normalsize
Note that since $\Delta_j$ can be smaller than $0$ for some arm $j$, we can further upper bound $R^{B}\left(\algKUBE\right)$ by only summing up $\Delta_{j}\E\left[N_{j}\left(T\right) | T\right]$ over arms with $\Delta_j > 0$ (see the last two inequalities).  
Applying Lemma~\ref{lemma:KUBE_pulls_small_budget_KUBE} to the first term and 
Lemma~\ref{lemma:KUBE_pull_number_j_KUBE} to the second term on the right hand side of Equation~\ref{eq:KUBE_theorem3_1} gives:
\begin{flalign}
\label{eq:KUBE_theorem3_2}
\nonumber
R^{B}\left(\algKUBE\right) &\leq \left[\left(\frac{8}{d^2_{\mathrm{min}}} + \left(\frac{\cmax}{\cmin}\right)^2\right)\sum_{\delta_j > 0}{\frac{\delta_j}{c_{I^{*}}}} \ln{\left(\frac{B}{\cmin}\right)} 
 + \sum_{\delta_j > 0}{\frac{\delta_j}{c_{I^{*}}}}\left( \frac{\pi^2}{3} + 1 \right) + 1\right]\mu_{I^{*}} + \\
\nonumber
&+ \E_{T}\left[\sum_{\Delta_j > 0}{\Delta_j}\left(\left(\frac{8}{d^2_{\mathrm{min}}} + \left( \frac{\cmax}{\cmin}\right)^2 \right)\ln{\left(T\right)} + \frac{\pi^2}{3} + 1\right)\right] \\
\nonumber
&\leq \left(\frac{8}{d^2_{\mathrm{min}}} + \left(\frac{\cmax}{\cmin}\right)^2\right)\sum_{\delta_j > 0}{\frac{\delta_j}{c_{I^{*}}}} \ln{\left(\frac{B}{\cmin}\right)} 
 + \sum_{\delta_j > 0}{\frac{\delta_j}{c_{I^{*}}}}\left( \frac{\pi^2}{3} + 1 \right) + 1 + \\
\nonumber
&+ \sum_{\Delta_j > 0}{\Delta_j}\left(\left(\frac{8}{d^2_{\mathrm{min}}} + \left( \frac{\cmax}{\cmin}\right)^2 \right)\ln{\left(\frac{B}{\cmin}\right)} + \frac{\pi^2}{3} + 1\right) 
\end{flalign}
\normalsize
which concludes the proof.
Note that the last equation is obtained from the facts that $\mu_{I^{*}} \leq 1$ and $T \leq \frac{B}{\cmin}$ with probability $1$.
\end{proof}

In a similar vein, we can show that the regret of fractional \algKUBE\ is bounded as follows:
\begin{theorem}[Main result 2]
\label{theorem:fractional_KUBE_regret_upper_bound}
For any budget size $B > 0$, the performance regret of fractional \algKUBE\ is at most
\small
\begin{eqnarray*}
\label{eq:fractional_KUBE_regret_upper_bound}
\frac{8}{d^2_{\mathrm{min}}}\left(\sum_{\Delta_j > 0}{\Delta_j} 
 + \sum_{\delta_j > 0}{\frac{\delta_j}{c_{I^{*}}}}\right)\ln{\left(\frac{B}{\cmin}\right)}
 + \left(\sum_{\Delta_j > 0}{\Delta_j} 
 + \sum_{\delta_j > 0}{\frac{\delta_j}{c_{I^{*}}}}\right)\left( \frac{\pi^2}{3} + 1 \right)+1\,.
\end{eqnarray*}
\normalsize
\end{theorem}
%
%
\begin{proof}{of Theorem~\ref{theorem:fractional_KUBE_regret_upper_bound}}
We follow the concept that is similar to the proof of Theorem~\ref{theorem:KUBE_regret_upper_bound_KUBE}.
Given this, we only highlight the steps that are different from the previous proofs.  
For the sake of simplicity, we use the notations previously introduced for the performance analysis of \algKUBE. 
In particular, let $T$ denote the random variable that represents the number of pulls that fractional \algKUBE\ uses.
Let $N_{j}\left(T\right)$ denote the number of times that the corresponding pulling algorithm pulls arm $j$ up to time step $T$.
Similar to Lemma~\ref{lemma:KUBE_pull_number_j_KUBE}, we first show that within the fractional \algKUBE\ algorithm, we have:
\begin{equation}
\label{eq:fractional_KUBE_1}
\E\left[N_{j}\left(T\right)| T\right] \leq \frac{8}{d^2_{\mathrm{min}}}\ln{\left(T\right)} + \frac{\pi^2}{3} + 1.
\end{equation}
\normalsize
In so doing, note that
\begin{flalign}
\label{eq:fractional_KUBE_2}
\E\left[N_{j}\left(T\right)| T\right] &= 1 + \sum_{t=K+1}^{T}{P\left(i\left(t\right) = j | T\right)} \leq  l + \sum_{t=K+1}^{T}{P\left(i\left(t\right) = j, N_{j}\left(t\right) \geq l | T\right)} 
\end{flalign}
\normalsize 
for any $l \geq 1$.
Now, using similar techniques from the proof of Lemma~\ref{lemma:KUBE_pull_number_j_KUBE}, we can easily show that
\begin{equation*}
\sum_{t=K+1}^{T}{P\left(i\left(t\right) = j, N_{j}\left(t\right) \geq l | T\right)} 
\leq \sum_{t=1}^{T}{\sum_{s=1}^{t}{\sum_{s_j = 1}^{t}{2t^{-4}} }} \leq \frac{\pi^2}{3},
\end{equation*}
\normalsize
for any $l \geq \left\lceil \frac{8\ln{T}}{d^{2}_{\mathrm{min}}} \right\rceil$.
By substituting this into Equation~\ref{eq:fractional_KUBE_2}, we obtain Equation~\ref{eq:fractional_KUBE_1}.
Next, we show that
\begin{equation}
\label{eq:fractional_KUBE_3}
\E\left[T\right] \geq \frac{B}{c_{I^{*}}} 
 - \frac{8}{d^2_{\mathrm{min}}}\sum_{\delta_j > 0}{\frac{\delta_j}{c_{I^{*}}}} \ln{\left(\frac{B}{\cmin}\right)} 
 - \sum_{\delta_j > 0}{\frac{\delta_j}{c_{I^{*}}}}\left( \frac{\pi^2}{3} + 1 \right) - 1. 
\end{equation}
\normalsize
This can be derived from Equation~\ref{eq:fractional_KUBE_1} by using techniques similar to the proof of Lemma~\ref{lemma:KUBE_pulls_small_budget_KUBE}.
This implies that
\begin{flalign}
\label{eq:fractional_KUBE_4}
\nonumber
R^{B}\left(\algKUBE\right) 
&= \E\left[G^{B}\left(A^{*}\right)\right] - \E\left[G^{B}\left(\algKUBE\right)\right] \\
\nonumber
&\leq \frac{B\mu_{I^{*}}}{c_{I^{*}}} 
 - \E_{T,\{i\left(t\right)\}}\left[ \sum_{t=1}^{T}{\mu_{i\left(t\right)}} \right] \\ 
\nonumber
&\leq \frac{B\mu_{I^{*}}}{c_{I^{*}}} - \E_{T}\left[ \sum_{t=1}^{T}{ \E_{i\left(t\right)}\left[\mu_{i\left(t\right)}\right] } \right] \\
\nonumber
&\leq \E_{T}\left[ \frac{B\mu_{I^{*}}}{c_{I^{*}}}  - \sum_{t=1}^{T}{ \E_{i\left(t\right)}\left[\mu_{i\left(t\right)}\right] } \right] 
\end{flalign}
\begin{flalign}
\nonumber
&\leq \E_{T}\left[ \frac{B\mu_{I^{*}}}{c_{I^{*}}} 
 - \sum_{t=1}^{T}{\sum_{j}^{K}{\mu_{j}P\left(i\left(t\right)=j | T\right)}  } \right] \\
\nonumber
&\leq \E_{T}\left[ \left(\frac{B}{c_{I^{*}}} - T\right)\mu_{I^{*}} + \sum_{t=1}^{T}{\left( \mu_{I^{*}} - \sum_{j}^{K}{\mu_{j}P\left(i\left(t\right)=j | T\right)} \right) } \right] \\
\nonumber
&\leq \E_{T}\left[\frac{B}{c_{I^{*}}} - T\right]\mu_{I^{*}} + \E_{T}\left[\sum_{t=1}^{T}{\sum_{\Delta_j > 0}{\Delta_{j}P\left(i\left(t\right)=j | T\right)}} \right] \\
&\leq \E_{T}\left[\frac{B}{c_{I^{*}}} - T\right]\mu_{I^{*}}
 + \E_{T}\left[\sum_{\Delta_j > 0}{\Delta_{j}\E\left[N_{j}\left(T\right) | T\right]} \right].
\end{flalign} 
\normalsize
By substituting Equations~\ref{eq:fractional_KUBE_2} and~\ref{eq:fractional_KUBE_3} into this, we obtain
\begin{flalign}
\label{eq:fractional_KUBE_5}
\nonumber
R^{B}\left(\algKUBE\right) &\leq \frac{8}{d^2_{\mathrm{min}}}\sum_{\delta_j > 0}{\frac{\delta_j}{c_{I^{*}}}} \ln{\left(\frac{B}{\cmin}\right)} 
 + \sum_{\delta_j > 0}{\frac{\delta_j}{c_{I^{*}}}}\left( \frac{\pi^2}{3} + 1 \right) + 1 + \\
\nonumber
&+ \sum_{\Delta_j > 0}{\Delta_j}\left(\frac{8}{d^2_{\mathrm{min}}}\ln{\left(\frac{B}{\cmin}\right)} + \frac{\pi^2}{3} + 1\right) 
\end{flalign}
\normalsize
which concludes the proof.
\end{proof}

Having established a regret bound for the two algorithms, 
we now move on to show that they produce optimal behaviour, in terms of minimising the regret.
In more detail, we state that:
\begin{theorem}[Main result 3]
\label{theorem:optimal_bound} 
For any arm pulling algorithm, there exists a constant $C \geq 0$, and a particular instance of the budget--limited MAB problem, such that the regret of that algorithm within that particular problem is at least $C\ln{B}$.      
\end{theorem}
\begin{proof}{of Theorem~\ref{theorem:optimal_bound}}
By setting all of the arms' pulling costs equal to $c \geq 0$, 
any standard MAB problem can be reduced to a budget--limited MAB.
This implies that the number of pulls within this MAB is guaranteed to be $\frac{B}{c} = T$ (i.e. $T$ is deterministic).
According to \citep{LaiAndRobbins1985}, 
the best possible regret that an arm pulling algorithm can achieve within the domain of standard MABs is 
$C\ln{\left(T\right)}$. 
Therefore, if there is an algorithm within the domain of budget--limited that provides better regret than 
$C\ln{\left(\frac{B}{c}\right)} = C\ln{T}$, then it also provides better regret bounds for standard MABs.
\end{proof}

The results in Theorem~\ref{theorem:KUBE_regret_upper_bound_KUBE} and~\ref{theorem:fractional_KUBE_regret_upper_bound} can be interpreted to the standard MAB domain as follows.
The standard MAB can be reduced to a budget--limited MAB by setting all the pulling costs to be the same. 
Given this, $\nicefrac{B}{\cmin} = T$ in any sequence of pulls. 
This implies that both \alg\ and fractional \alg\ achieve $O\left(\ln{T}\right)$ regret within the standard MAB domain, which is optimal \citep{LaiAndRobbins1985,AuerEtal2002}.

Note that the regret bound of fractional \alg\ is better (i.e. the constant factor within the regret bound of fractional \alg\ is smaller than that of \alg).
However, this does not indicate that fractional \alg\ has better performance in practice.
One possible reason is that these bounds are not tight. 
In fact, as we will demonstrate in Section~\ref{Section:numerical}, \alg\ typically outperforms its fractional counterpart by up to $40\%$.

\section{Performance Evaluation}
\label{Section:numerical}

\begin{figure}[t]
{\centering
\resizebox*{0.7\textwidth}{!}{
\psfrag{Epsilon-first 05}[cl][cl][1.2][0]{$\varepsilon$--first ($\varepsilon = 0.05$) }
\psfrag{Epsilon-first 10}[cl][cl][1.2][0]{$\varepsilon$--first ($\varepsilon = 0.10$) }
\psfrag{Epsilon-first 15}[cl][cl][1.2][0]{$\varepsilon$--first ($\varepsilon = 0.15$) }
\psfrag{KUBE}[cl][cl][1.0][0]{\alg\ }
\psfrag{fractional KUBE}[cl][cl][1.0][0]{Fractional \alg\ }
\psfrag{lnB}[cl][cl][1.0][0]{$O(B^{\frac{2}{3}}\left(\ln{B}\right)^{-1})$}
\includegraphics{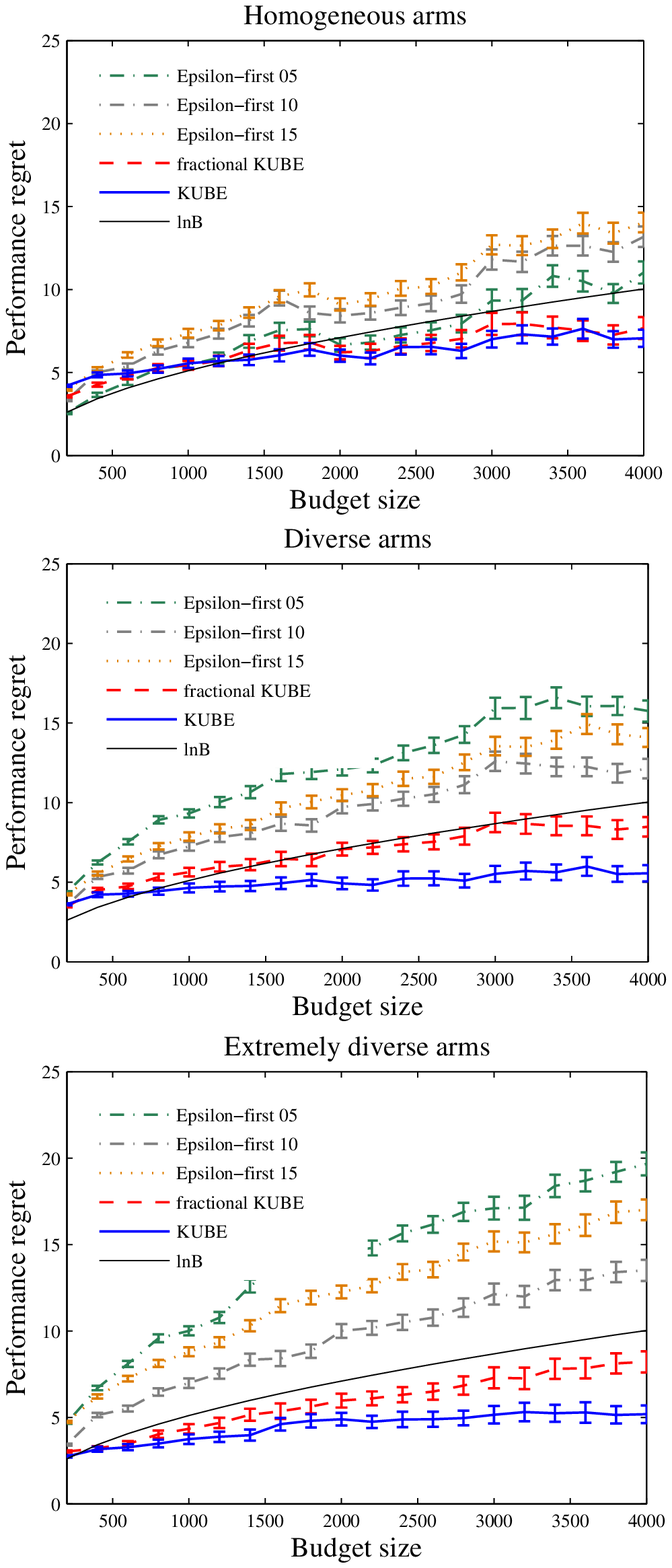}
}
\par}
\caption{\small\textit{Performance regret of the algorithms, divided by $\ln\left(\frac{B}{c_{min}}\right)$, for a $100$--armed bandit machine with homogeneous arms, moderately diverse arms, or extremely diverse arms (left to right).
\normalsize}}
\label{fig:num_results_all}
\end{figure}

\noindent 
In the previous section, we showed that the two algorithms provide asymptotically optimal regret bounds, and that the theoretical regret bound of fractional \alg\ is tighter than that of \alg.
In addition, we also demonstrated that fractional \alg\ outperforms \alg\ in terms of computational complexity.
However, it might be the case that these bounds are not tight, and thus, fractional \alg\ is less practical than \alg\ in real--world applications, as is the case with the standard MAB algorithm, where simple but not optimal methods (e.g. $\varepsilon$--first, or $\varepsilon$--greedy) typically outperform more advanced, theoretically optimal, algorithms (e.g. POKER\citep{VermorelAndMohri2005}, or UCB). 
Given this, we now evaluate the performance of both algorithms through extensive simulations, in order to determine their efficiency in practice.
We also compare the performance of the proposed algorithms against that of different budget--limited $\varepsilon$--first approaches.
In particular, we show that both of our algorithms outperform the budget--limited $\varepsilon$--first algorithms.
In addition, we also demonstrate that \alg\ typically achieves lower regret than its fractional counterpart.

Now, note that if the pulling costs are homogeneous --- that is, the pulling cost of the arms do not significantly differ from each other --- then the performance of the density--ordered greedy algorithm does not significantly differ from that of the fractional relaxation based \citep{KellererEtAl2004}.
Indeed, since the pulling costs are similar, it is easy to show that the density--ordered greedy approach typically stops after one round, and thus, results in similar behaviour to the fractional relaxation based method.
On the other hand, if the pulling costs are more diverse (i.e. the pulling costs of the arms differ from each other), then the performance of the density--ordered greedy algorithm becomes more efficient than that of the fractional relaxation based algorithm.
Given this, in order to compare the performance of \alg\ and its fractional counterpart, we set three test cases, namely: bandits with (i) homogeneous pulling costs; (ii) moderately diverse pulling costs; and (iii) extremely diverse costs.  
In particular, within the homogeneous case, the pulling costs are randomly and independently chosen from the interval $\left[5,10\right]$. 
In addition, the pulling costs are set to be between $\left[1,10\right]$ within the moderately diverse case, and between $\left[1,20\right]$ in the extremely diverse case, respectively. 
The reward distribution of each arm $i$ is set to be a truncated Gaussian, with mean $\mu_{i}$, randomly taken from interval $\left[10,20\right]$, variance $\sigma_{i}^{2} = \frac{\mu_{i}}{2}$, and with supports $[0, 2\mu_{i}]$.
In addition, we set number of arms $K$ to be $100$.
 
Our results are shown in Figure~\ref{fig:num_results_all}. 
These plots show the performance of each algorithm divided by $\ln{\frac{B}{c_{\mathrm{min}}}}$, 
and the error bars represent the $95\%$ confidence intervals.
By doing this, we can see that the performance regret of both algorithms is \small$O\left(\ln{\frac{B}{c_{\mathrm{min}}}}\right)$\normalsize, since in each test case, their performance converges to \small$C\ln{\frac{B}{c_{\mathrm{min}}}}$\normalsize  
(after it is divided by \small$\ln{\frac{B}{c_{\mathrm{min}}}}$\normalsize), where $C$ is some constant factor.
From the numerical results, we can see that both \alg\ and fractional \alg\ differ from the best possible solution by small constant factors (i.e. $C$), since the limit of their convergence is typically low (i.e. it varies between $4$ and $7$ in the test cases), compared to the regret value of the algorithm.
In addition, we can also see that fractional \alg\ algorithm is typically outperformed by \alg.
The reason is that the density--ordered greedy algorithm provides a better approximation than the fractional relaxation based  approach to the underlying unbounded knapsack problem.
This implies that \alg\ converges to the optimal pulling policy faster than its fractional counterpart.
In particular, as expected, the performance of the algorithms are similar to each other in the homogeneous case, where the density--ordered greedy method shows similar behaviour to the fractional relaxation based approach. 
In contrast, \alg\ clearly achieves better performance (i.e. lower regret) within the diverse cases.
Specifically, within the moderately diverse case, \alg\ outperforms its fractional counterpart by up to $40\%$ (i.e. the regret of \alg\ is $40\%$ lower than that of the fractional \alg\ algorithm).
In addition, the performance improvement of \alg\ is typically around $30\%$ in the extremely diverse case.
This implies that, although the current theoretical regret bounds are asymptotically optimal, they are not tight.


Apart from this, we can also observe that both of our algorithms outperform the budget--limited $\varepsilon$--first approaches. 
In particular, \alg\ and its fractional counterpart typically achieves less regret by up to $70\%$ and $50\%$ than the budget--limited $\varepsilon$--first approaches, respectively. 
Note that the performance of the proposed algorithms are typically under the line \small$O(B^{\frac{2}{3}}\left(\ln{B}\right)^{-1})$\normalsize, while the budget--limited $\varepsilon$--first approaches achieve larger regrets.
This implies that our proposed algorithms are the first methods that achieve logarithmic regret bounds.

\section{Conclusions}
\label{Section:conclusion}

\noindent 
In this paper, we introduced two new algorithms, \alg\ and fractional \alg, for the budget--limited MAB problem.
These algorithms sample each arm in an initial phase.
Then, at each subsequent time step, they determine a best set of arms, 
according to the agent's current reward estimates plus a confidence interval based on the number of samples taken of each arm. 
In particular, \alg\ uses the density--ordered greedy algorithm to determine this best set of arms.
In contrast, fractional \alg\ relies on the fractional relaxation based algorithm.
\alg\ and its fractional counterpart then use this best set as a probability distribution with which to randomly choose the next arm to pull. 
As such, both algorithms do not explicitly separate exploration from exploitation.
We have also provided a $O\ln\left(B\right)$ theoretical upper bound for the performance regret of both algorithms, where $B$ is the budget limit.
In addition, we proved that the provided bounds are asymptotically optimal, 
that is, they differ from the best possible regret by only a constant factor.
Finally, through simulation, we have demonstrated that \alg\ typically outperforms its fractional counterpart up to $40\%$, however, with an increased computational cost.
In particular, the average computational complexity of \alg\ per time step is $O\left(K\ln{K}\right)$, while this value is $O\left(K\right)$ for fractional \alg.
 
One of the implications of the numerical results is that although fractional \alg\ has a better bound on its performance regret than \alg, the latter typically ourperforms the former in practice.
Given this, our future work consists of improving the results of Theorems~\ref{theorem:KUBE_regret_upper_bound_KUBE} and~\ref{theorem:fractional_KUBE_regret_upper_bound} to determine tighter upper bounds can be found. 
In addition, we aim to extend the budget--limited MAB model to settings where the reward distributions are dynamically changing, as is the case in a numer of real--world problems.
This, however, is not trivial, since both of our algorithms rely on the assumption that the expected value of the rewards is static, and thus, the estimates converge to their real value.

\bibliographystyle{natbib}
\bibliography{PhD_thesis_bib}

\begin{thebibliography}{15}
\expandafter\ifx\csname natexlab\endcsname\relax\def\natexlab#1{#1}\fi
\expandafter\ifx\csname url\endcsname\relax
  \def\url#1{{\tt #1}}\fi

\bibitem[Antos \emph{et~al}.(2008)Antos, Grover, and Szepesv\'{a}ri]{Antos2008}
A.~Antos, V.~Grover, and Cs. Szepesv\'{a}ri.
\newblock Active learning in multi-armed bandits.
\newblock {\em In Proceedings of the Nineteenth International Conference on
  Algorithmic Learning Theory}, pages 287--302, 2008.

\bibitem[Audibert \emph{et~al}.(2009)Audibert, Munos, and
  Szepesv\'{a}ri]{AudibertEtAl2009}
J-Y. Audibert, R.~Munos, and Cs. Szepesv\'{a}ri.
\newblock Exploration-exploitation trade-off using variance estimates in
  multi-armed bandits.
\newblock {\em Theoretical Computer Science}, 410:\penalty0 1876--1902, 2009.

\bibitem[Auer \emph{et~al}.(2002)Auer, Cesa-Bianchi, and Fischer]{AuerEtal2002}
P.~Auer, N.~Cesa-Bianchi, and P.~Fischer.
\newblock Finite--time analysis of the multiarmed bandit problem.
\newblock {\em Machine Learning}, 47:\penalty0 235--256, 2002.

\bibitem[Bubeck \emph{et~al}.(2009)Bubeck, Munos, and Stoltz]{BubeckEtal2009}
S.~Bubeck, R.~Munos, and G.~Stoltz.
\newblock Pure exploration for multi-armed bandit problems.
\newblock {\em In Proceedings of the Twentieth international conference on
  Algorithmic Learning Theory}, pages 23--37, 2009.

\bibitem[Guha and Munagala(2007)]{GuhaMunagala2007}
S.~Guha and K.~Munagala.
\newblock Approximation algorithms for budgeted learning problems.
\newblock {\em In Proceedings of the Thirty-Ninth Annual ACM symposium on
  Theory of Computing}, pages 104--113, 2007.

\bibitem[Hoeffding(1963)]{Hoeffding1963}
W.~Hoeffding.
\newblock Probability inequalities for sums of bounded random variables.
\newblock {\em ournal of the American Statistical Association}, 58:\penalty0
  13--30, 1963.

\bibitem[Ivic(1985)]{Ivic1985}
A.~Ivic, editor.
\newblock {\em The Riemann Zeta Function}.
\newblock John Wiley \& Sons, 1985.
\newblock ISBN 0-471-80634-X.

\bibitem[Kellerer \emph{et~al}.(2004)Kellerer, Pferschy, and
  Pisinger]{KellererEtAl2004}
H.~Kellerer, U.~Pferschy, and D.~Pisinger.
\newblock {\em Knapsack Problems}.
\newblock Springer, 2004.

\bibitem[Kohli \emph{et~al}.(2004)Kohli, Krishnamurti, and
  Mirchandani]{KohliEtal2004}
R.~Kohli, R.~Krishnamurti, and P.~Mirchandani.
\newblock Average performance of greedy heuristics for the integer knapsack
  problem.
\newblock {\em European Journal of Operational Research}, 154(1):\penalty0
  36--45, 2004.

\bibitem[Lai and Robbins(1985)]{LaiAndRobbins1985}
T.~L. Lai and H.~Robbins.
\newblock Asymptotically efficient adaptive allocation rules.
\newblock {\em Advances in Applied Mathemathics}, 6(1):\penalty0 4--22, 1985.

\bibitem[Padhy \emph{et~al}.(2010)Padhy, Dash, Martinez, and
  Jennings]{PadhyEtAl2010}
P.~Padhy, R.~K. Dash, K.~Martinez, and N.~R. Jennings.
\newblock A utility-based adaptive sensing and multihop communication protocol
  for wireless sensor networks.
\newblock {\em ACM Transactions on Sensor Networks}, 6(3):\penalty0 1--39,
  2010.

\bibitem[Robbins(1952)]{Robbins1952}
H.~Robbins.
\newblock Some aspects of the sequential design of experiments.
\newblock {\em Bulletin of the AMS}, 55:\penalty0 527--535, 1952.

\bibitem[Tran-Thanh \emph{et~al}.(2010)Tran-Thanh, Chapman, de~Cote, Rogers,
  and Jennings]{Tran-ThanhEtAl2010}
L.~Tran-Thanh, A.~Chapman, J.~E.~Munoz de~Cote, A.~Rogers, and N.~R. Jennings.
\newblock Epsilon--first policies for budget--limited multi--armed bandits.
\newblock {\em In Proceedings of the Twenty-Fourth Conference on Artificial
  Intelligence}, pages 1211--1216, 2010.

\bibitem[Tran-Thanh \emph{et~al}.(2011)Tran-Thanh, Rogers, and
  Jennings]{Tran-ThanhEtAl2011}
L.~Tran-Thanh, A.~Rogers, and N.~R. Jennings.
\newblock Long--term information collection with energy harvesting wireless
  sensors: A multi–armed bandit based approach.
\newblock {\em Journal of Autonomous Agents and Multi-Agent Systems}, 2011.

\bibitem[Vermorel and Mohri(2005)]{VermorelAndMohri2005}
J.~Vermorel and M.~Mohri.
\newblock Multi-armed bandit algorithms and empirical evaluation.
\newblock {\em European Conference on Machine Learning}, pages 437--448, 2005.

\end{thebibliography}
\end{document}